\documentclass[journal]{IEEEtran}
\IEEEoverridecommandlockouts
\usepackage{cite}
\usepackage{amsmath,amssymb,amsfonts}
\usepackage{graphicx}
\usepackage{textcomp}
\usepackage{hyperref}
\usepackage{booktabs}
\usepackage{adjustbox}
\usepackage{subcaption}
\usepackage{algorithm}
\usepackage{algpseudocode}
\usepackage{amsthm}
\usepackage{comment}
\usepackage{multirow}
\usepackage[table]{xcolor}
\theoremstyle{plain}
\newtheorem{theorem}{Theorem}
\newtheorem{proposition}{Proposition}
\newtheorem{definition}{Definition}
\newtheorem{criterion}{Criterion}
\def\Cinf{{C^\infty}}
\def\C{{\cal C}}

\def\Cfree{{\cal C}_{free}}

\def\RN{\mathbb{R}^{n}}

\def\BibTeX{{\rm B\kern-.05em{\sc i\kern-.025em b}\kern-.08em
    T\kern-.1667em\lower.7ex\hbox{E}\kern-.125emX}}
\begin{document}

\begin{center}
\small
\copyright~2026 IEEE. This work has been accepted for publication in IEEE Robotics and Automation Letters.
Personal use of this material is permitted. Permission from IEEE must be obtained for all other uses.
\end{center}

\title{Smooth Feedback Motion Planning with Reduced Curvature}

\author{Aref Amiri and Steven M. LaValle
\thanks{This work was supported by Infotech Oulu, a European Research Council Advanced Grant, and the Academy of Finland (project BANG! 363637)}%
\thanks{The authors are with the Faculty of Information Technology and Electrical Engineering, University of Oulu, Oulu, Finland
{\tt\small firstname.lastname@oulu.fi}}

}

\maketitle

\begin{abstract}
Feedback motion planning over cell decompositions provides a robust method for generating collision-free robot motion with formal guarantees. However, existing algorithms often produce paths with unnecessary bending, leading to slower motion and higher control effort. This paper presents a computationally efficient method to mitigate this issue for a given simplicial decomposition. A heuristic is introduced that systematically aligns and assigns local vector fields to produce more direct trajectories, complemented by a novel geometric algorithm that constructs a maximal star-shaped chain of simplexes around the goal. This creates a large ``funnel'' in which an optimal, direct-to-goal control law can be safely applied. Simulations demonstrate that our method generates measurably more direct paths, reducing total bending by an average of 91.40\% and LQR control effort by an average of 45.47\%. Furthermore, comparative analysis against sampling-based and optimization-based planners confirms the time efficacy and robustness of our approach. While the proposed algorithms work over any finite-dimensional simplicial complex embedded in the collision-free subset of the configuration space, the practical application focuses on low-dimensional ($d\le3$) configuration spaces, where simplicial decomposition is computationally tractable.
\end{abstract}


\section{Introduction}

Feedback motion planning synthesizes a vector field over $\Cfree$, the collision-free subspace of the configuration space~\cite{lavalle2006planning,163777,tedrake2010lqrtrees}. This creates a closed-loop policy that robustly guides a robot to its goal from almost any admissible state, an essential feature for real-world applications~\cite{lindemann2009simple}. The framework introduced by Lindemann et al. provides a powerful method for generating such feedback plans by decomposing the space into cells and blending local vector fields~\cite{lindemann2005smoothly,lindemann2006real,lindemann2009simple}. This approach constructs almost globally defined, $\Cinf$-smooth feedback laws that guarantee convergence and collision avoidance, nicely sidestepping the local minima that plague artificial potential fields~\cite{khatib1986real} and the high computational cost of global methods such as harmonic functions~\cite{connolly1990path}.

The primary limitation of Lindemann et al. is that it lacks the ability to assign local vector fields in a way that generates higher-quality paths~\cite{lindemann2009simple}. The default assignments of local cell and face vectors, while satisfying theoretical convergence criteria, can produce integral curves with unnecessary bending, particularly when moving between cells. These inefficient paths lead to longer travel times, higher energy consumption, and increased control effort, diminishing the practical utility of the feedback law.

This work presents a computationally efficient solution to this problem by systematically assigning local vector fields to produce more direct trajectories. We introduce a simple and efficient heuristic for aligning cell vector fields; by propagating a desired global direction of motion backward from the goal, our method systematically biases local vector fields to produce more direct integral curves. This is enhanced by the face-vector averaging technique that smooths transitions between cells, directly reducing path bending. Furthermore, we develop a novel geometric algorithm to construct a maximal star-shaped chain of simplexes around the goal, creating a large, geometrically verified ``funnel'' where a direct-to-goal control law can be safely applied for optimal, straight-line convergence. Demonstrations in maze, bug trap, and sparse environments show that our method generates qualitatively superior feedback laws, with trajectories exhibiting measurably lower total bending and reduced control effort compared to the baseline, all while preserving the foundational guarantees of smoothness and global convergence. Although simulations are demonstrated in two dimensions based on Constrained Delaunay Triangulation (CDT)~\cite{shewchuk1996triangle}, the algorithm applies to any simplicial complex embedded in $\Cfree$. While obtaining such complexes in higher dimensions is challenging, it is a tractable problem for 2D/3D configuration or workspace.  
\section{Related Work}
Our work is based on established methods in feedback motion planning and computational geometry. We situate our contributions in relation to these key areas: classical potential-based methods, feedback laws on cell decompositions, and modern optimization-based and region-based planning.
\subsection{Potential Fields, Navigation and Harmonic Functions}

Artificial Potential Fields (APFs)~\cite{khatib1986real} offer a fast and reactive feedback policy, but suffer from local minima, oscillation, and poor performance in narrow passages~\cite{koren1991potential}. Therefore, enhancements have been proposed to mitigate these issues~\cite{weerakoon2015artificial}. Navigation Functions (NFs)~\cite{163777} are a special class of smooth APFs that ensure global convergence to a goal without local minima under certain assumptions. NFs have been extended to more general geometries~\cite{6225105} or for non-holonomic and multiple robots~\cite{1307519}. However, they are practically challenging to implement, especially in environments with complex obstacle geometries~\cite{lindemann2009simple}. Harmonic functions~\cite{connolly1993applications,connolly1990path} also overcome local minima by solving Laplace's equation over the free space. However, the computational cost of solving partial differential equations (PDEs) restricts their use in real-time or high-dimensional applications. 

\subsection{Feedback from Cell Decompositions and Bump Functions}

To overcome the limitations of global potential functions, Lindemann et al.~\cite{lindemann2005smoothly,lindemann2009simple} proposed constructing local vector fields over a convex cell decomposition. In this framework, local vector fields are blended using smooth bump functions to create a feedback law with guaranteed convergence. While this approach is general and has been extended to shaped robots~\cite{zhang2009global} and nonholonomic systems~\cite{lindemann2006real}, the assignment of the local vector fields to reduce path bending remains a critical open challenge that this paper addresses. Bump function blends have also been used to compose attractive/repulsive fields with strong safety guarantees~\cite{panagou2014motion,panagou2016distributed}. 
\subsection{Region-Based and Optimization-Based Planning}

In sampling-based methods such as RRT*~\cite{karaman2011sampling}, post-processing techniques such as B-spline or shortcutting are common to improve path quality~\cite{ravankar2018path,geraerts2007creating}. While these operate on a single path, our method, by construction, generates a $\Cinf$-smooth vector field and therefore integral curves. Advanced optimization-based methods, such as the Graph of Convex Sets (GCS)~\cite{Marcucci2024SPP} and Fast Path Planning (FPP)~\cite{marcucci2024fast}, find smooth and high-quality paths by solving mathematical programs over convex decompositions. In contrast, our geometric approach avoids expensive optimization, constructing a smooth, computationally efficient feedback law over $\mathcal{C}_{free}$. Tedrake et al.~\cite{tedrake2010lqrtrees} developed the LQR-Trees framework, where stabilizing controllers are verified using sum-of-squares optimization to produce funnel-shaped regions of attraction around sampled trajectories. Later work by Reist et al.~\cite{reist2016feedback} replaced the verification process with a simulation-based expansion, broadening the range of applicable systems. While these methods are powerful and optimal, they rely on computationally intensive procedures like sums-of-squares or semidefinite programming to derive regions of attraction.

In motion planning, identifying large subsets of free space where simple control laws can safely guide the robot is a key strategy. A well-established method is to compute convex regions, either through exact decomposition or iterative expansion. For example, the IRIS algorithm~\cite{deits2015computing} expands large convex polytopes via alternating semidefinite and quadratic programs.

Other works have explored maximum convex region extraction in triangulated spaces~\cite{aronov2007largest}, where convex unions of adjacent triangles are found via dynamic programming. Although convexity simplifies controller design, it is often restrictive. This motivates using star-shaped regions, which we construct geometrically, avoiding dynamics-based verification~\cite{tedrake2010lqrtrees,reist2016feedback}.

Our geometric construction of a maximal star-shaped region is conceptually similar to the ``Triangular Expansion'' algorithm used for computing visibility polygons~\cite{bungiu2014efficient}, as both recursively explore adjacent triangles. However, a key difference lies in our objective. Instead of computing the exact boundary of the visibility polygon, our goal is simply to find the maximal set of simplexes that form a star-shaped region relative to the goal. This simpler objective allows for a more efficient visibility check that generalizes naturally to higher dimensions.
\section{Problem Formulation}

We consider the motion planning problem for a robot navigating in a workspace $W \subset \RN$ (where $n \in \{2, 3\}$) with static obstacles $\mathcal{O}$. The robot's configuration is a point in its $d$-dimensional configuration space, $\C$, and all motion is restricted to the open collision-free subset, $\Cfree$.
 
The task is to design a $\Cinf$-smooth vector field $V$ over $\Cfree$ whose integral curves are guaranteed to converge to $x_{g}\in \Cfree$ without collision. We note that $V$ is well-defined and smooth almost everywhere, except on a set of measure zero (the ($d-2$)-dimensional boundaries of the cell faces, e.g., vertices in 2D) which the integral curves never traverse~\cite{lindemann2009simple}. To manage geometric complexity, a simplicial complex composed of a finite number of non-degenerate $d$-simplexes, $\mathcal{T}=\{\Delta_{1},...,\Delta_{N}\}$, is embedded in $\Cfree$. In this paper's simulations, we use a point robot in 2D; hence, $\C = W \subset\mathbb{R}^{2}$ and obstacles are polygonal, and by employing CDT, we ensure that every $2$-simplex is non-degenerate. 

We compute a discrete plan on $\mathcal{T}$ (with vertices at simplex centroids and edges connecting adjacent simplexes) by finding a shortest-path tree to the goal simplex $\Delta_g$ containing $x_g$. This provides a successor mapping, $s(i)$, for each simplex $\Delta_i \in \mathcal{T}$ that guides the robot toward the goal.

The global vector field is constructed from two local components for each simplex $\Delta_i$: 1) a single, constant cell vector field, $V_{c,i}$, respecting successor mapping, $s(i)$, and 2) a set of face vector fields, $\{V_f\}$, defined on the boundary faces of the simplex. The final vector field $V(x)$ at any point $x \in \Delta_i$ is a smooth, weighted combination of $V_{c,i}$ and the vector field of the closest face, following the blending method in~\cite{lindemann2009simple}.

To further improve efficiency, we identify a maximal star-shaped chain of simplexes $\mathcal{C}_{g} \subseteq \mathcal{T}$ around the goal that forms a star-shaped region $R_g = \bigcup_{\Delta \in \mathcal{C}_g} \Delta$. A set $R$ is star-shaped with respect to $x_g \in R$ if for every $x \in R$, the line segment that connects $x$ and $x_g$ is contained in $R$. Within this ``funnel'' region $R_g$, the control law is simplified: the cell vector field $V_{c,i}$ for every $\Delta_i \in \mathcal{C}_g$ is set to point directly toward the goal $V_{c,i}(x) = \text{normalize}(x_g - x)$, where $\text{normalize}(v) = v / \|v\|$ for any non-zero vector $v$.

\section{Smooth Feedback on Simplicial Decompositions}

Our approach is based on the framework of~\cite{lindemann2009simple} for simplicial decomposition. The method involves computing a global vector field from local vector fields defined on individual cells and their faces, ensuring smoothness, collision avoidance, and global convergence to a designated goal state.

To precisely describe our method, we first recall several essential definitions and theorems from~\cite{lindemann2009simple}.

\subsection{Fundamental Definitions}
To obtain smooth interpolation between local vector fields, we use a special class of smooth functions $b(\sigma)$, called bump functions~\cite{lindemann2009simple}, understood here as $\Cinf$ functions that provide a smooth transition between 0 and 1.  
\begin{definition}%[Bump Function]
\label{def: def1}
The function $b: \mathbb{R} \to [0, 1]$ is a $\Cinf$ smooth function defined as:
\begin{equation}
b(\sigma) = \begin{cases}
0 & \sigma \leq 0, \\
\frac{\lambda(\sigma)}{\lambda(\sigma) + \lambda(1 - \sigma)} & 0 < \sigma < 1, \\
1 & \sigma \geq 1,
\end{cases}
\end{equation}
where the auxiliary function $\lambda(\sigma)$ is defined as:
\begin{equation}
\lambda(\sigma) = \frac{1}{\sigma} e^{-1/\sigma},
\end{equation}
with the property that all derivatives of $\lambda(\sigma)$ are zero at the endpoints ($\sigma=0$ and $\sigma=1$) and therefore any derivative of $b(\sigma)$ is zero at these endpoints. %This function smoothly transitions from $0$ to $1$ and is essential for interpolating vector fields across cell boundaries.
\end{definition}
To guide the robot through the cell, we define a cell vector that guarantees reaching the exit face (a shared face between the current cell and its successor cell). To formally define the region of influence for each ($d-1$)-dimensional face, we refer to the Generalized Voronoi Diagram (GVD)~\cite{88035}. Within a convex set, the GVD is the set of all points equidistant to at least two faces, which divides the cell into distinct regions of influence for each face.
\begin{definition}%[Cell Vector Field]
\label{def: def2}
For a convex cell $\Delta_i$ with an exit face $f_{\text{exit},i}$, a cell vector field $V_{c,i}$ is a smooth unit vector field on $\Delta_i$ that satisfies three conditions:
\begin{enumerate}
    \item For each point $x \in \Delta_i$, there exists $y \in f_{\text{exit},i}$ and $\alpha \in \mathbb{R}$ such that $V_{c,i}(x) = \alpha(y - x)$.
    \item Let $h$ be a GVD face with normal vector $n$. If $V_{c,i}(x) \cdot n = 0$ for some $x \in h$, then $V_{c,i}(x) \cdot n = 0$ for all $x \in h$.
    \item The directed transition graph induced by this choice of vector fields is acyclic, and every path through this graph terminates at the node corresponding to the exit face.
\end{enumerate}
\end{definition}
To ensure that the robot crosses the exit face and avoids collisions with obstacles, we define face vector fields.

\begin{definition}%[Face Vector Field]
\label{def: def3}
A face vector field $V_f$ corresponding to a face $f$ of a cell $\Delta_i$ is a smooth unit vector field satisfying:
\begin{itemize}
    \item For each point $p \in f$, the vector field $V_f(p)$ satisfies $V_f(p) \cdot n > 0$, where $n$ is the inward-pointing normal vector for $f \in F_i \setminus \{f_{\text{exit},i}\}$, and outward-pointing normal vector for $f = f_{\text{exit},i}$.
    \item For each non-exit face $f \in F_i \setminus \{f_{\text{exit},i}\}$, let $b_f$ denote the hyperplane equidistant from $f$ and $f_{\text{exit},i}$ with unit normal $n_{b_f}$ oriented such that $n_{b_f} \cdot n_x > 0$, where $n_x$ is the outward normal of $f_{\text{exit},i}$. Then $V_f(p) \cdot n_{b_f} > 0$ holds for every $p$ in the closure of the region of influence of $f$.
\end{itemize}
\end{definition}

\subsection{Essential Theorems}

The framework of~\cite{lindemann2009simple} provides the following guarantees, which our method inherits (detailed proofs of these theorems can be found there):

\begin{theorem}%[Finite Exit Time]    
All integral curves generated by the smoothly interpolated cell and face vector fields within any intermediate cell reach the designated exit face of that cell in finite time.
\end{theorem}
\begin{theorem}%[Goal Cell Convergence]
All integral curves within the goal cell asymptotically converge to the designated goal point.
\end{theorem}
\begin{theorem}%[Global Convergence]
The integral curves defined over the entire decomposed cells asymptotically converge to the goal state from any initial state located within the decomposed cells.
\end{theorem}
\begin{theorem}%[Global Smoothness]
The integral curves generated by the combined local vector fields remain $C^\infty$-smooth throughout the entire domain, including across the GVD surface within each cell and face boundaries.
\end{theorem}

While the general framework allows for vector fields converging to a specific point, we utilize constant cell vector fields for simplicity and uniformity. This choice prevents trajectories from bunching together around the designated point, such as the exit face centroid, and simplifies the alignment process. In our method, we impose boundary constraints (boundary vectors) for the constant cell vector fields. Each $d$-simplex has $d+1$ vertices, and each $(d-1)$-dimensional face has $d$ vertices corresponding to it~\cite{edelsbrunner2001geometry}, meaning that for each $(d-1)$-dimensional face, there exists one opposite vertex. Therefore, when selecting an exit face for the cell, a corresponding opposite vertex exists (Let us call it $v_{ov}$). We define boundary vectors that connect this $v_{ov}$ vertex to each of the vertices of the exit face. The cell vector field must be chosen within the wedge defined by these $d$ boundary vectors (for $d$-simplex). Geometrically, this results in a conical region oriented toward the exit face.  Constraining the constant vector field to this conical region is sufficient to guarantee the validity conditions of Definition~\ref{def: def2}.
\begin{theorem}
\label{thm:thm5}
 A constant cell vector field $V_{c,i}$ that is chosen to point within the conical region defined by the boundary vectors satisfies the conditions of Definition~\ref{def: def2}.
\end{theorem}
\begin{proof}
 Let $\mathcal{K}_i$ be the conical region. Any constant vector $V_{c,i}$ within $\mathcal{K}_i$ can be expressed as a positive linear combination of the boundary vectors. Every boundary vector $b_{i,j}$  points from the opposite vertex $v_{ov,i}$ towards a point on the exit face $f_{\text{exit},i}$. Therefore, any positive linear combination of these vectors, and thus $V_{c,i}$, will point towards $f_{\text{exit},i}$. Let $h$ be a face of the GVD with a constant normal vector $n$. The dot product is $V_{c,i} \cdot n$. Since both $V_{c,i}$ and $n$ are constant vectors, their dot product is a constant value everywhere. Therefore, if $V_{c,i} \cdot n = 0$ for some point on $h$, it is zero for all points on $h$. Since $V_{c,i}$ points towards the exit face, any trajectory starting in $\Delta_i$ will proceed monotonically toward $f_{\text{exit},i}$ and is guaranteed to intersect it. The path cannot curve back on itself or form a cycle. This ensures that the induced graph of the transitions is acyclic.
\end{proof}
Hence, selecting a constant vector field within this conical region guarantees its validity under the framework of~\cite{lindemann2009simple}.

\section{Heuristic Vector Field Alignment Method}

Arbitrary vector field assignments, even those satisfying Theorem~\ref{thm:thm5}, can lead to unnecessary bending in integral curves. To mitigate this, we introduce a heuristic alignment strategy that adjusts each cell vector field to better align with the vector field of its successor(s), producing more direct trajectories.
\subsection{ Heuristic Alignment of Cell Vector Fields}
First, we establish a global plan. Let $x_g$ be the goal point, and let $\Delta_g$ be the goal simplex. For each simplex $\Delta_i$, let $c_i$ be its centroid. We compute the shortest-path tree rooted at $\Delta_g$ on the connectivity graph, which provides a successor $s(i)$ and a hop-distance $l(i)$ (i.e., the minimum number of simplex transitions required to reach the goal simplex $\Delta_g$) for each simplex $\Delta_i$. For each non-goal cell $\Delta_i$, we identify its exit face $f_{\text{exit},i}$ and the corresponding opposite vertex $v_{ov,i}$, as illustrated in Figure~\ref{fig:heuristic_alignment}.
 \begin{figure}[h]
     \centering
     \includegraphics[width=0.7\columnwidth,trim={4cm 1cm 4cm 1cm},clip]{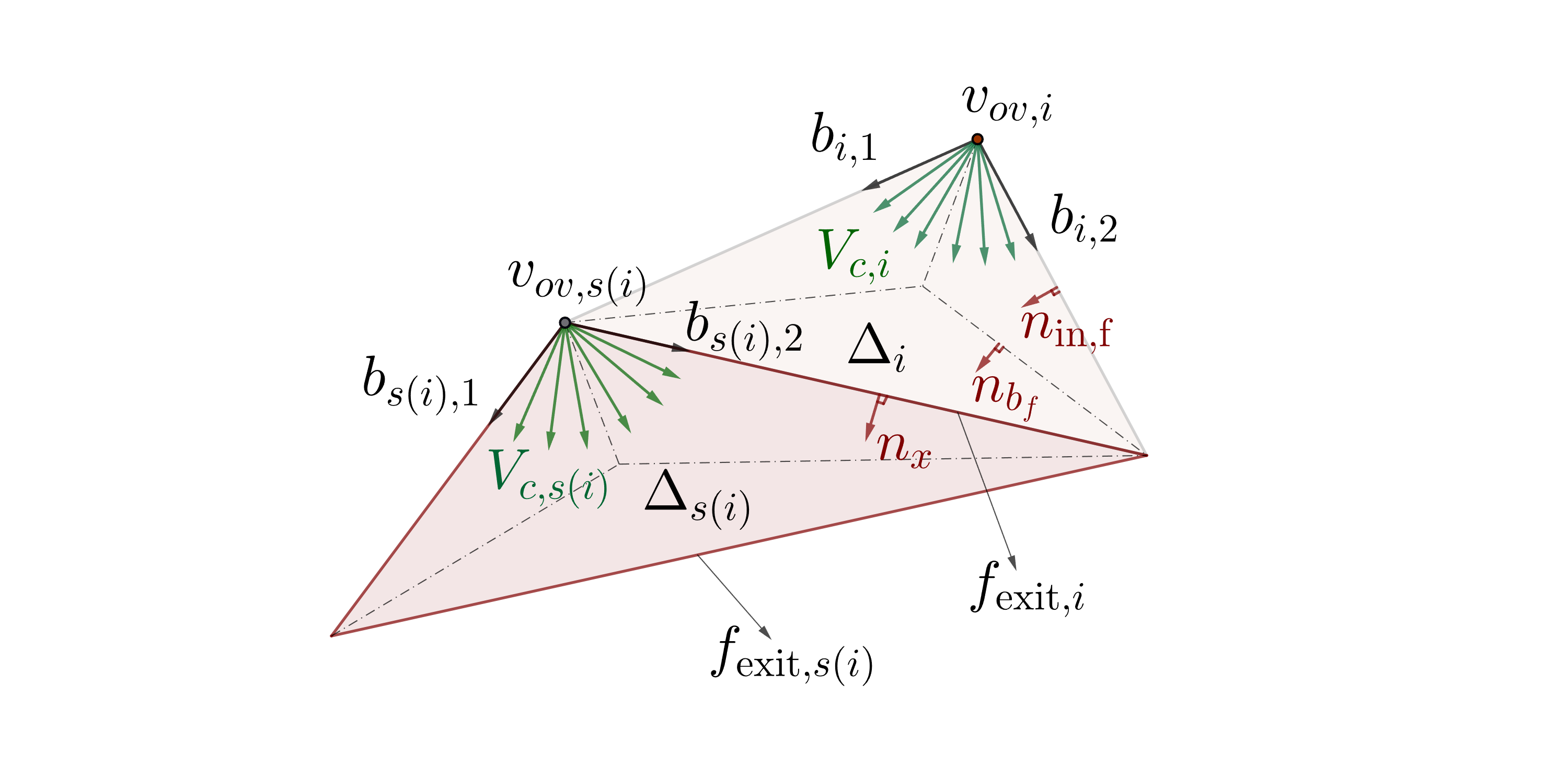}
     
     \caption{Geometric construction of valid vector fields for 2-simplex (triangle). The conical region, constructed by boundary vectors $b_{i,j}$, illustrates the set of valid constant cell vector fields $V_{c,i}$ that satisfy the conditions of Definition~\ref{def: def2}. The diagram visualizes the geometric constraints for face vector fields (Definition~\ref{def: def3}), showing the face inward normal $n_{\text{in},f}$, the exit face outward normal $n_{x}$, and the hyperplane normal $n_{b_f}$. The dashed lines are the GVD surface.}
     \label{fig:heuristic_alignment}
 \end{figure}

The core of our heuristic is to process cells in order of increasing hop-distance from the goal, propagating a desired direction backward. For each non-goal cell $\Delta_i$, we first define its $d$ unit boundary vectors $\{b_{i,j}\}$ which form its conical region $\mathcal{K}_i$:
\begin{equation}
 b_{i,j} = \text{normalize}(p_{v_j} - p_{v_{ov,i}}), \quad j=1,...,d ,
\end{equation}
in which $p_{v_j}$ are the coordinates of the vertices of the exit face $f_{\text{exit},i}$. We then assign the cell vector field $V_{c,i}$ as Algorithm~\ref{alg: alg_1}.

\begin{algorithm}
\caption{Heuristic Alignment of Cell Vector Fields}
\label{alg: alg_1}
\small
\begin{algorithmic}[1]
\Require Set of simplexes $\mathcal{T}$, goal point $x_g$.
\Ensure Cell vector field $V_{c,i}$ for all $\Delta_i \in \mathcal{T}$.
\State Compute centroids $c_i$, graph $G$, successors $s(i)$, and hop-distances $l(i)$ for all $\Delta_i$.
\For{each non-goal simplex $\Delta_i$ in ascending order of $l(i)$}
    \State Let $\{b_{i,j}\}$ be the set of boundary vectors forming the conical region $\mathcal{K}_i$.
    \If{$l(i) = 1$}
        \State $v_{\text{cand},i} \gets x_g - c_i$
    \Else
        \State $v_{\text{cand},i} \gets V_{c,s(i)}$
    \EndIf
    \State Normalize $v_{\text{cand},i}$.
    \If{$v_{\text{cand},i}$ is not in the conical region $\mathcal{K}_i$}
        \State Find $b_{i,j^*}$ that minimizes the angle with $v_{\text{cand},i}$.
        \State $V_{c,i} \gets b_{i,j^*}$
    \Else
        \State $V_{c,i} \gets v_{\text{cand},i}$
    \EndIf
\EndFor
\State \Return $\{V_{c,i}\}$
\end{algorithmic}
\end{algorithm}

To formally check if a candidate vector $v_{\text{cand},i}$ lies within the conical region formed by the basis vectors $\{b_{i,j}\}$, we must determine if it can be expressed as a non-negative linear combination of them. This can be formulated as a standard linear programming (LP) feasibility problem, where we seek to find coefficients $\alpha_j \ge 0$ such that:
\begin{equation}
    v_{\text{cand},i}=\sum_{j=1}^{d}\alpha_{j}b_{i,j}.
\end{equation}
If a solution for $\{\alpha_j\}$ exists, the vector is inside the conical region. If this feasibility problem has no solution, the candidate vector lies outside the cone. In this case, we project it to the closest boundary vector (the one minimizing the angle), a computationally simple choice that guarantees the vector remains within the valid conical region. While this requires solving a small LP for each simplex, the problem is computationally inexpensive as the number of variables ($d$) equals the dimension of the space. This check is efficient enough for precomputation and does not impede the real-time evaluation of the feedback law.

 Although it does not guarantee a globally optimal path, in practice, it yields significantly straighter and more uniform trajectories compared to unaligned cell fields~\cite{lindemann2009simple}. The effectiveness of the alignment heuristic extends beyond local smoothing. The heuristic encourages consecutive cell vector fields, $V_{c,i}$ and $V_{c,s(i)}$, to adopt a similar direction. By then defining the exit face vector field as the average of these two vectors, our method naturally creates conditions where adjacent cells with a coherent flow functionally merge. This process forms larger regions of consistent flow, guiding the robot along globally straighter paths.   
 
\subsection{Assignment of Face Vector Fields and Proof of Validity}

\begin{itemize}
    \item For each non-exit face $f \in F_i \setminus \{f_{\text{exit},i}\}$, we modify the standard inward-pointing normal to better align with the cell's flow. The face vector $V_f$ is defined as the normalized sum of the face's inward-pointing unit normal, $n_{\text{in},f}$, and the cell's own vector field, $V_{c,i}$:
    \begin{equation}
    V_f = \text{normalize}(n_{\text{in},f} + V_{c,i}).
    \end{equation}

    \item  For the unique exit face $f_{\text{exit},i}$, we average the cell vectors of the current cell $\Delta_{i}$ and its successor $\Delta_{j}$ (where $j=s(i)$) to promote a smooth passage between cells:
    \begin{equation}
    V_{f}(f_{\text{exit},i}) = \text{normalize}(V_{c,i} + V_{c,j}).
    \end{equation}

\end{itemize}
We must ensure these assignments satisfy Definition~\ref{def: def3}, guaranteeing the robot exits strictly through the exit face.

 \begin{proposition}
 \label{pro:pro1}
 The modified non-exit face vector field $V_f = \mathrm{normalize}(n_{\text{in},f} + V_{c,i})$ satisfies the conditions of Definition~\ref{def: def3} for a non-exit face.
\end{proposition}
\begin{proof} First, we show $V_f\cdot n_{\text{in},f} > 0$. The dot product of the unnormalized sum, $(n_{\text{in},f} + V_{c,i}) \cdot n_{\text{in},f}$, expands to $1 + (V_{c,i} \cdot n_{\text{in},f})$ since $n_{\text{in},f}$ is a unit vector. As $V_{c,i}$ points toward the exit face restricted to be within its conical region, its dot product with the inward normal of any other face is non-negative ($V_{c,i} \cdot n_{\text{in},f} \ge 0$). The entire expression is therefore $\ge 1$, satisfying the first condition. Since the dot product is strictly positive, the vector sum $(n_{\text{in},f} + V_{c,i})$ cannot be the zero vector, and thus the normalization is always well-defined.

Second, we must show $V_f\cdot n_{b_f} > 0$, where $n_{b_f} = n_x + n_{\text{in},f}$ is the normal to the equidistant hyperplane. We prove that $(n_{\text{in},f} + V_{c,i})\cdot n_{b_f} >0$ by analyzing two components: 1) The cell vector component, $V_{c,i} \cdot n_{b_f}$, which expands to $(V_{c,i} \cdot n_x) + (V_{c,i} \cdot n_{\text{in},f})$. The first term, $V_{c,i} \cdot n_x$, is strictly positive because the conical region for $V_{c,i}$ is oriented toward the exit face. The second term, $V_{c,i} \cdot n_{\text{in},f}$, is non-negative because $V_{c,i}$ does not point into any non-exit faces. The sum of a strictly positive and a non-negative term is therefore strictly positive. 2) The geometric component is $n_{\text{in},f}\cdot n_{b_f} = (n_{\text{in},f}\cdot n_x) + (n_{\text{in},f} \cdot n_{\text{in},f})$. Since $n_{\text{in},f}$ is a unit vector, this simplifies to $(n_{\text{in},f} \cdot n_x) +1$. For any two distinct faces of a simplex, their respective inward and outward unit normals cannot be anti-parallel, meaning their dot product is strictly greater than $-1$. This guarantees the term is strictly positive. The sum of these two strictly positive components is therefore strictly positive, and the proposition holds.    
\end{proof}

\begin{proposition}
The averaged exit face vector field $V_f(f_{\text{exit},i}) = \mathrm{normalize}(V_{c,i} + V_{c,j})$ satisfies the conditions of Definition~\ref{def: def3} for an exit face.
\end{proposition}

\begin{proof} We show that $V_f(f_{\text{exit},i}) \cdot n_x > 0$. As established in the proof of Proposition~\ref{pro:pro1}, since $V_{c,i}$ is chosen from its conical region $\mathcal{K}_i$, it is guaranteed that $V_{c,i} \cdot n_x > 0$. Now consider the successor's cell vector, $V_{c,j}$. This vector is chosen to be within the conical region $\mathcal{K}_j$ of cell $\Delta_j$, which points towards its own exit face, $f_{\text{exit},j}$. This only guarantees that $V_{c,j}$ points away from its entrance face ($f_{\text{exit},i}$). Therefore, its component along $n_x$ is non-negative: $V_{c,j} \cdot n_x \ge 0$. %The equality holds in the worst-case scenario where $V_{c,j}$ is orthogonal to $n_x$.
The dot product of the sum is $(V_{c,i} + V_{c,j}) \cdot n_x = (V_{c,i} \cdot n_x) + (V_{c,j} \cdot n_x)$. The sum of a strictly positive and a non-negative term is strictly positive, satisfying the condition.
    \end{proof}

\vspace{-5mm}
\subsection{ Blending the Vector Fields}

Finally, within each cell $\Delta_i$, a single smooth vector field $V(x)$ is synthesized by interpolating between the cell vector $V_{c,i}$ and the vector of the closest face.

For any point $x \in \Delta_i$, we first identify its closest face, $f^* \in F_i$, by computing the perpendicular distance $d(x, f)$ to every face $f$ of the cell. The region of influence of $f^*$ (its Voronoi region) consists of all points for which $f^*$ is the closest face.

Within this region, we use the smooth bump function $b(\sigma)$ from Definition~\ref{def: def1}. The parameter $\sigma \in [0, 1]$ must be 0 on the face $f^*$ and 1 on the GVD surface (the boundary of the region of influence). The formula to achieve this is:
\begin{equation}
 \sigma(x) = 1 - \prod_{f \in F_i \setminus \{f^*\}}\left(\frac{d(x, f) - d(x, f^*)}{d(x, f)}\right).
\end{equation}

Let $V_{f^*}$ be the face vector for the closest face $f^*$. The final vector field at point $x$ is the normalized, weighted average: 
\begin{equation}
 V(x) = \text{normalize}\left( (1 - b(\sigma(x))) V_{f^*} + b(\sigma(x)) V_{c,i} \right).
\end{equation}

This construction guarantees that the vector field is identical to the face vector on the cell boundary, ensuring continuity across cells, and smoothly transitions to the cell vector field toward the interior of the cell.

\section{Constructing the Maximal Star-shaped Chain of Simplexes}

To create more direct trajectories, we identify a large funnel where the robot can safely bypass conservative cell-by-cell transitions. While standard algorithms can compute the exact visibility polytope, they are often computationally intensive as they must construct a precise geometric boundary. Our simplex-based expansion is therefore chosen for its efficiency and simplicity; it relies on a series of simple local visibility tests to incrementally grow the funnel.  We acknowledge that this construction is conservative and the resulting region $\mathcal{R}_g$ is strictly contained within the true visibility polytope of the goal. The expansion is sensitive to the simplicial complex; for example, Steiner points introduced for mesh quality can prevent the chain expansion even where a simplex is completely visible to the goal. However, this approach is sufficient for our purpose and naturally extends to higher dimensions.

\subsection{ Geometric Criteria for Star-shaped}

Our method incrementally grows a star-shaped region, $\mathcal{R}_g$, starting from the goal simplex, $\Delta_g$. Consider a region $\mathcal{R} \subset \mathbb{R}^d$ known to be star-shaped with respect to $x_g$. We wish to add an adjacent $d$-simplex, $\Delta_T$, which shares a $(d-1)$-dimensional face $\mathcal{S}$ with $\mathcal{R}$'s boundary. The union $\mathcal{R} \cup \Delta_T$ remains star-shaped if the new vertex of $\Delta_T$ not on $\mathcal{S}$, denoted $v_{\text{new}}$, is visible from $x_g$. See Figure~\ref{fig2} for an illustration. We propose two equivalent geometric criteria for this test.

\begin{criterion}[Visibility Cone from Goal]
\label{cri: cri2}
 A visibility cone is formed at the goal $x_g$ by the vectors pointing from $x_g$ to the vertices of the shared face $\mathcal{S}$. The region $\mathcal{R} \cup \Delta_T$ is star-shaped if the vector $v_{\text{new}} - x_g$ lies strictly inside this cone.
\end{criterion}

\begin{theorem}
Criterion~\ref{cri: cri2} guarantees preservation of the star-shaped property.
\end{theorem}

\begin{proof}
 Assume vector $v_{\text{new}} - x_g$ lies strictly within the visibility cone at $x_g$ formed by the vertices of face $\mathcal{S}$. The segment $[x_g, v_{\text{new}}]$ is, by definition, contained within this cone. Because region $\mathcal{R}$ is star-shaped, no other boundaries of $\mathcal{R}$ can obstruct this cone. The segment $[x_g, v_{\text{new}}]$ therefore passes unobstructed into $\Delta_T$. By convexity of $\Delta_T$, all segments from $x_g$ to any interior point of $\Delta_T$ remain fully within the extended region. Thus, $\mathcal{R} \cup \Delta_T$ is star-shaped. 
\end{proof}
Similar to the check in Section~V-A, this geometric condition can be formulated as a computationally inexpensive LP feasibility problem. We seek to determine if $v_{\text{new}}-x_g$ can be expressed as a non-negative linear combination of the vectors forming the visibility cone. 

\begin{figure}
 \centering
{\includegraphics[width=0.4\textwidth,trim={0 0.5cm 0 0.5cm},clip]{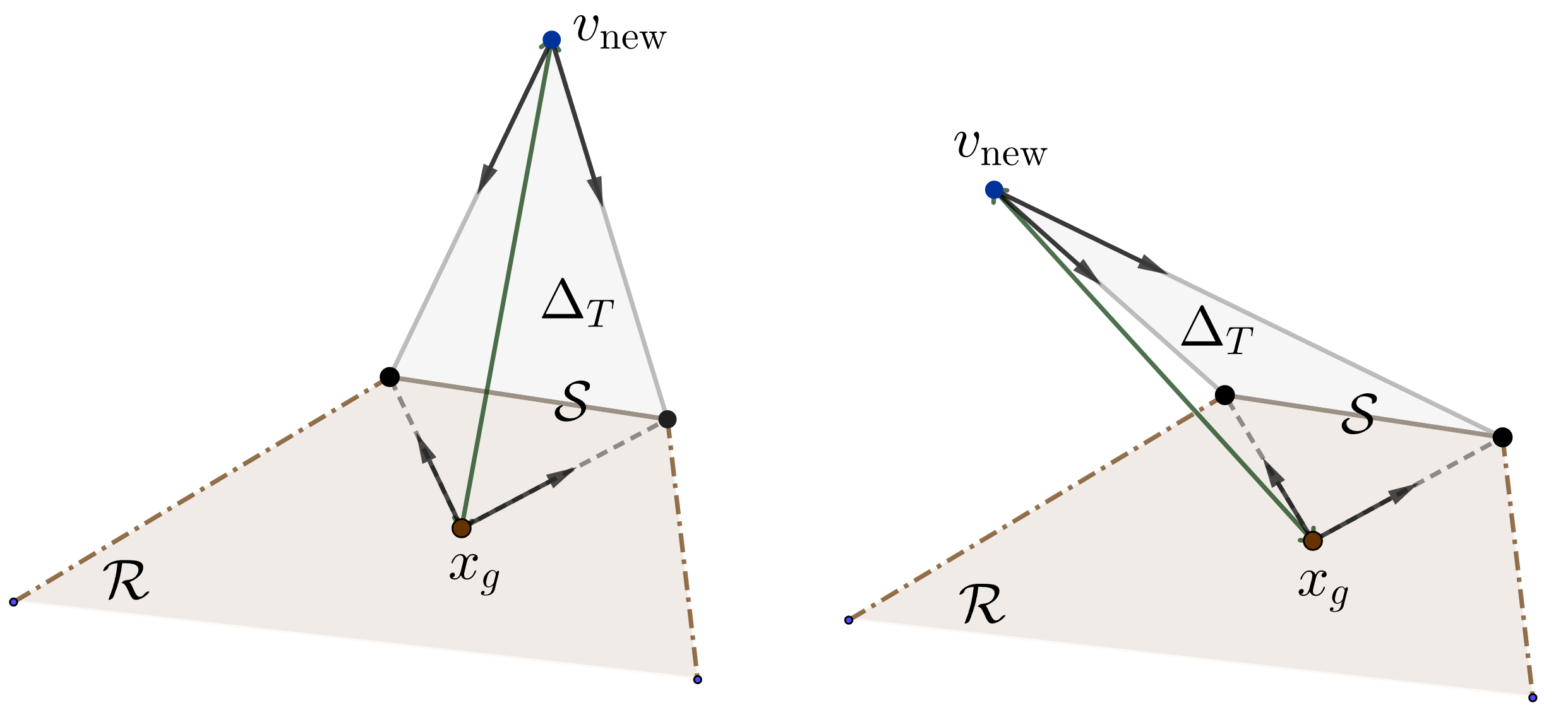}}
\caption{Geometric test for expanding the star-shaped region $\mathcal{R}$ with an adjacent simplex $\Delta_T$ (in 2D). (Left) A valid addition where the new vertex $v_\text{new}$ lies within the visibility cone from $x_g$ (Criterion~\ref{cri: cri2}). (Right) An invalid addition where visibility is obstructed.}
\label{fig2}
\end{figure}

\subsection{Algorithm for Growing the Star-Shaped Chain}

We find the maximal chain of simplexes, $\mathcal{C}_g$, that satisfy the visibility criterion by searching outward from the goal simplex $\Delta_g$. The general geometric search is detailed in Algorithm~\ref{alg: alg_2}. For our experiments, this search was constrained to only explore backward along the precomputed discrete plan to ensure a fair comparison with the baseline. Within the resulting region $\mathcal{R}_g = \bigcup_{\Delta_i \in \mathcal{C}_g} \Delta_i$, for any cell $\Delta_i \in \mathcal{C}_g$, its cell vector field $V_{c,i}$, as well as its internal face vectors (not on the boundary of $\mathcal{R}_g$), are set to point directly to the goal. To ensure a smooth transition at the boundary of $\mathcal{R}_g$, we keep the face vectors assigned to each face and smoothly blend them with the new direct-to-goal $V_{c,i}$. 

\begin{algorithm}
\caption{Construction of the Maximal Star-Shaped Chain}
\label{alg: alg_2}
\small
\begin{algorithmic}[1]
\Require Set of simplexes $\mathcal{T}$, goal point $x_g$.
\Ensure The set of simplexes $\mathcal{C}_g$ forming the star-shaped region.
\State Let $\Delta_g$ be the simplex containing $x_g$.
\State Initialize queue $Q$ with $\Delta_g$.
\State Initialize $\mathcal{C}_g = \{\Delta_g\}$ and a set of visited simplexes, $V_s = \{\Delta_g\}$.
\While{$Q$ is not empty}
    \State $\Delta_i \gets \text{dequeue}(Q)$.
    \For{each neighbor $\Delta_j$ of $\Delta_i$ not in $V_s$}
        \State Let $\mathcal{S}$ be the shared face between $\Delta_i$ and $\Delta_j$.
        \State Let $v_{\text{new}}$ be the vertex of $\Delta_j$ not on $\mathcal{S}$.
        \If{Vertex $v_{\text{new}}$ lies in the visibility cone from $x_g$ through $\mathcal{S}$}
            \State Add $\Delta_j$ to $\mathcal{C}_g$.
            \State Enqueue $\Delta_j$ into $Q$.
            \State Add $\Delta_j$ to $V_s$.
        \EndIf
    \EndFor
\EndWhile
\State \Return $\mathcal{C}_g$
\end{algorithmic}
\end{algorithm}

\section{Results and Analysis}
To evaluate the proposed method, we first compare its trajectory quality against the foundational baseline~\cite{lindemann2009simple}. We then conduct an ablation study to isolate and quantify the contribution of the star-shaped region. Finally, we compare our approach against four other motion planners and analyze its overall computational efficiency.
\subsection{Comparison with Baseline Feedback Law}
 The baseline method offers to simply assign each cell vector field as the unit vector pointing from the current point to the midpoint of the exit face, and defines face vector fields as unit vectors normal (perpendicular) to the faces, pointing inward for non-exit faces and outward for the exit face. The comparison was performed in three distinct environments, and to evaluate the robustness of the proposed method, we chose the centroid of every cell in the triangulation as a goal state, generating and analyzing a total of over 20000 integral curves for each environment.
We quantified path quality using a set of geometric and dynamics-based metrics.
For a path with curvature $\kappa(s)$, we measured the Total Bending, $E_B = \int_0^L \kappa(s)^2 ds$, and the Total Turning, $E_T = \int_0^L |\kappa(s)| ds$, which measures cumulative heading change. We also measured path length, maximum curvature, and dynamic feasibility by simulating a 2D double integrator with state vector $[p_x,\dot p_x,p_y,\dot p_y]^T$, using an LQR controller to find the travel time and total control effort with the cost matrices $Q = \text{diag}([100, 1, 100, 1])$ and $R = I_2$ to heavily penalize position error. 

The quantitative results are summarized in Table~\ref{tab:results}. The table includes a 'Win Rate' metric, which indicates the percentage of paired trajectories (starting from the same initial point) for which the proposed method achieved a better score on the given metric. For a fair and direct comparison, the underlying discrete plan was computed once and held constant for both methods.

\begin{table}[h!]
\centering
\caption{Quantitative Comparison of Methods. \\ $N_G$: Number of Goal Cells tested, $N_P$: Total Integral Curves evaluated.}
\label{tab:results}
\resizebox{\columnwidth}{!}{%
\begin{tabular}{@{}l l c c c c@{}}
\toprule
\textbf{Environment} & \textbf{Metric} & \textbf{Baseline} & \textbf{Proposed Method} & \textbf{Improv. (\%)} & \textbf{Win Rate (\%)} \\
\midrule
\multirow{6}{*}{\shortstack[l]{\textbf{Maze}\\ \small{$(N_G=109, N_P=52045)$}}}
& Path Length & 16.70 $\pm$ 10.85 & 12.14 $\pm$ 7.85 & 27.30 & 97.89 \\
& Max Curvature & 23.35 $\pm$ 8.66 & 18.38 $\pm$ 8.40 & 21.30 & 66.59 \\
& Total Bending & 602.65 $\pm$ 1664.75 & 94.44 $\pm$ 76.88 & 84.33 & 97.50 \\
& Total Turning & 47.28 $\pm$ 65.04 & 12.21 $\pm$ 9.03 & 74.18 & 99.73 \\
& LQR Travel Time & 16.70 $\pm$ 10.85 & 12.14 $\pm$ 7.85 & 27.30 & 97.29 \\
& LQR Control Effort & 19.59 $\pm$ 13.12 & 8.00 $\pm$ 4.96 & 59.16 & 98.55 \\
\midrule
\multirow{6}{*}{\shortstack[l]{\textbf{Bug Trap}\\ \small{$(N_G=145, N_P=20126)$}}}
& Path Length & 15.41 $\pm$ 10.12 & 13.70 $\pm$ 9.02 & 11.13 & 99.44 \\
& Max Curvature & 26.82 $\pm$ 6.51 & 14.19 $\pm$ 4.88 & 47.10 & 99.80 \\
& Total Bending & 1058.11 $\pm$ 744.96 & 43.61 $\pm$ 28.05 & 95.88 & 99.88 \\
& Total Turning & 82.26 $\pm$ 56.82 & 6.49 $\pm$ 3.99 & 92.11 & 99.89 \\
& LQR Travel Time & 15.41 $\pm$ 10.12 & 13.70 $\pm$ 9.02 & 11.13 & 99.11 \\
& LQR Control Effort & 9.48 $\pm$ 5.68 & 6.61 $\pm$ 3.96 & 30.24 & 97.93 \\
\midrule
\multirow{6}{*}{\shortstack[l]{\textbf{ Sparse}\\ \small{$(N_G=75, N_P=33200)$}}}
& Path Length & 7.07 $\pm$ 3.96 & 6.09 $\pm$ 3.33 & 13.84 & 92.43 \\
& Max Curvature & 14.85 $\pm$ 6.33 & 9.32 $\pm$ 6.85 & 37.21 & 79.17 \\
& Total Bending & 309.80 $\pm$ 2590.75 & 18.60 $\pm$ 19.70 & 94.00 & 97.52 \\
& Total Turning & 22.54 $\pm$ 99.41 & 3.80 $\pm$ 2.65 & 83.16 & 99.70 \\
& LQR Travel Time & 7.07 $\pm$ 3.96 & 6.09 $\pm$ 3.33 & 13.84 & 91.71 \\
& LQR Control Effort & 6.61 $\pm$ 3.18 & 3.50 $\pm$ 1.84 & 47.02 & 97.24 \\
\bottomrule
\end{tabular}
}
\end{table}

\begin{figure}[t!]
    \centering
    
   \begin{subfigure}[b]{0.48\columnwidth}
        \centering
        \includegraphics[width=\textwidth,trim={2.5cm 2.5cm 2.5cm 2.5cm},clip]{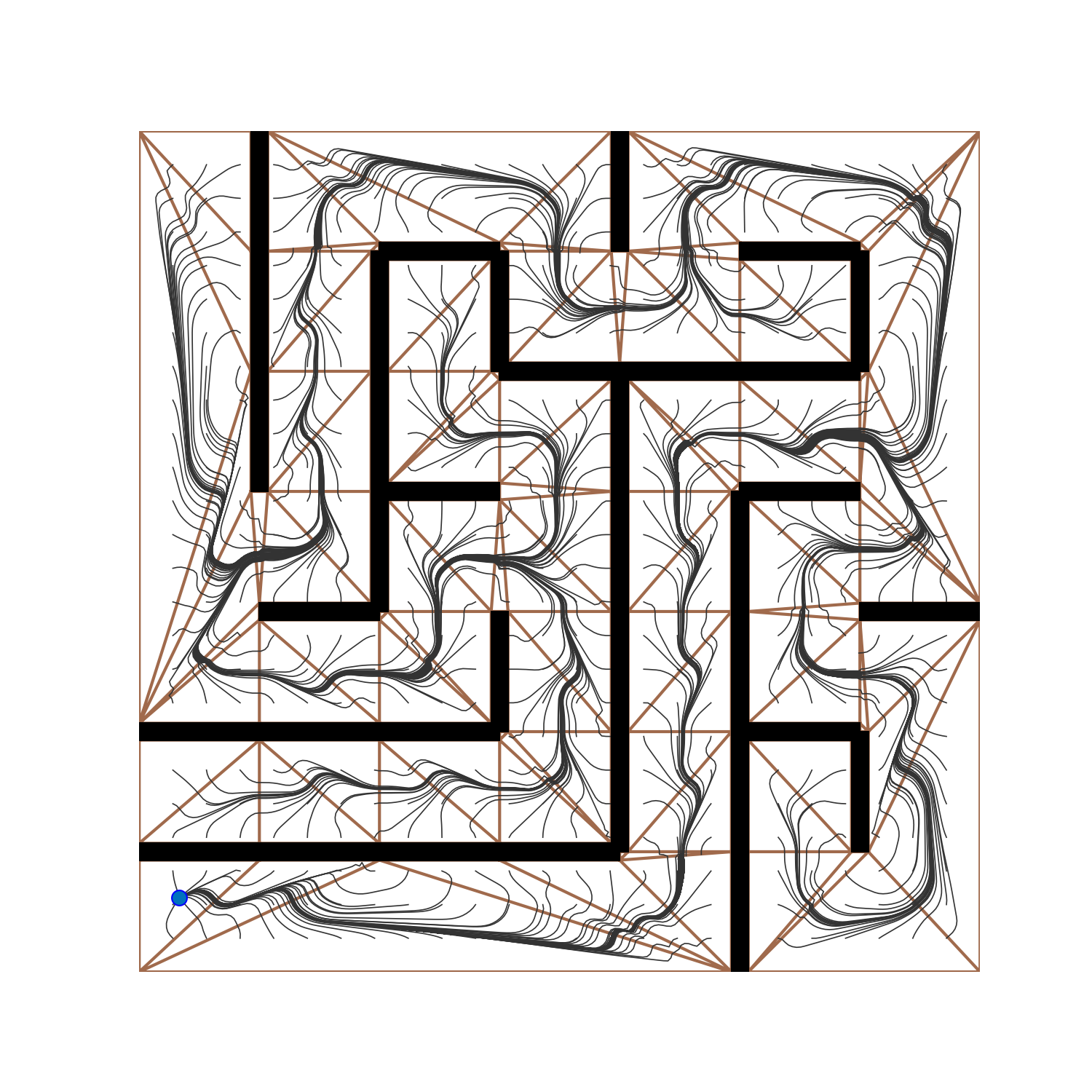}
        \caption{Baseline: Maze}
        \label{fig:maze_base}
    \end{subfigure}
    \hfill 
    \begin{subfigure}[b]{0.48\columnwidth}
        \centering
        \includegraphics[width=\textwidth,trim={2.5cm 2.5cm 2.5cm 2.5cm},clip]{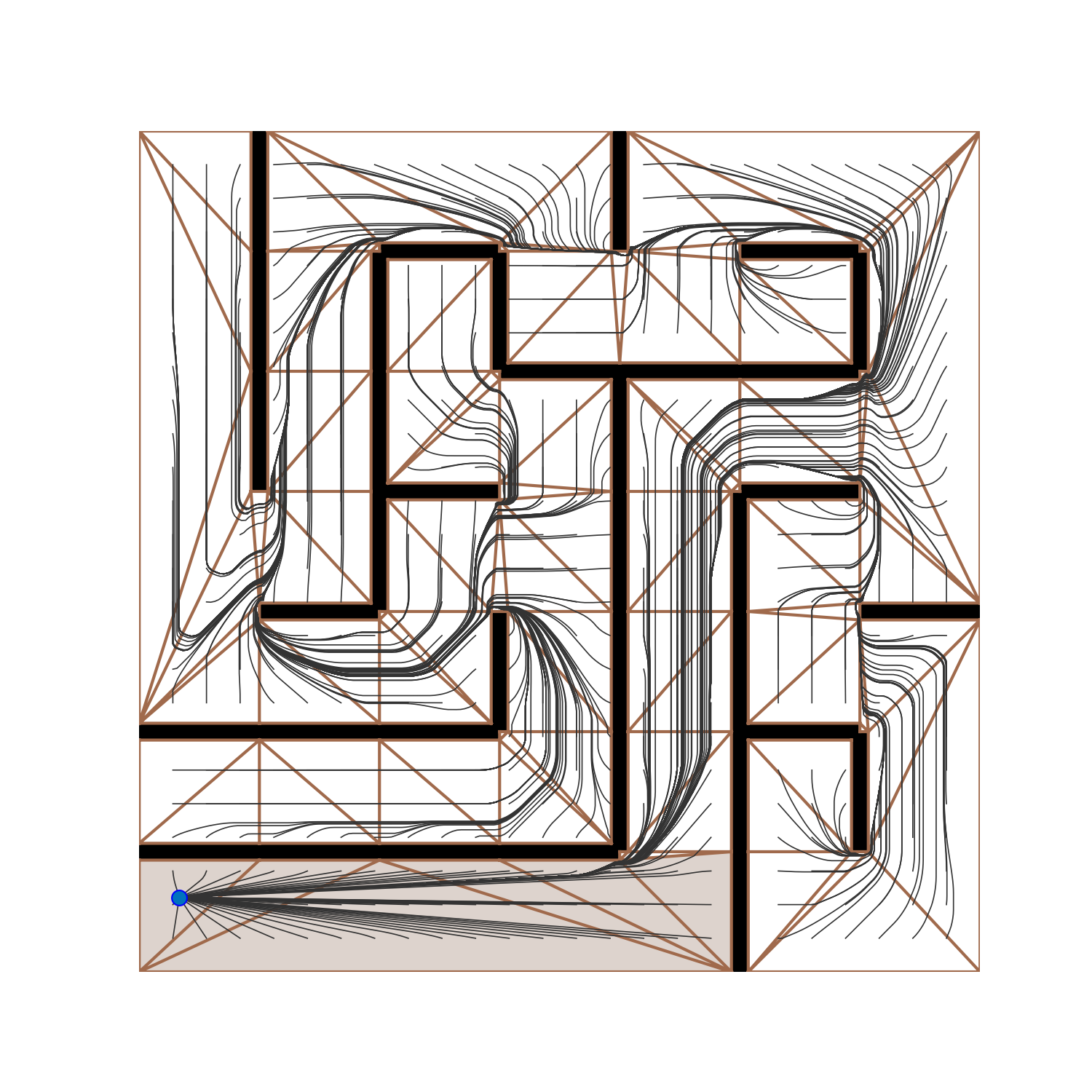}
        \caption{Proposed: Maze}
        \label{fig:maze_proposed}
    \end{subfigure}

    \begin{subfigure}[b]{0.48\columnwidth}
        \centering
        \includegraphics[width=\textwidth,trim={2.5cm 2.5cm 2.5cm 2.5cm},clip]{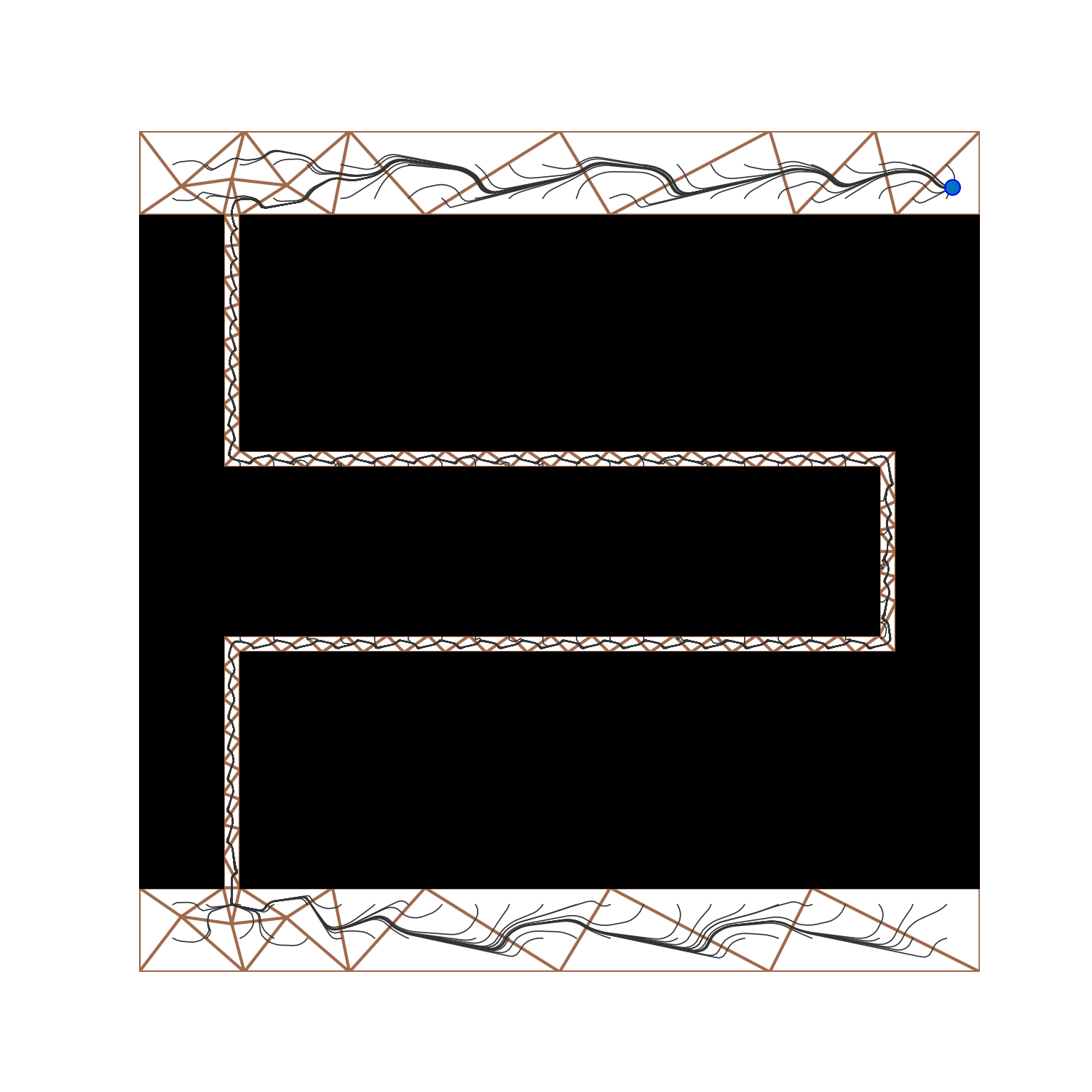}
        \caption{Baseline: Bug Trap}
        \label{fig:bug_base}
    \end{subfigure}
    \hfill 
    \begin{subfigure}[b]{0.48\columnwidth}
        \centering
        \includegraphics[width=\textwidth,trim={2.5cm 2.5cm 2.5cm 2.5cm},clip]{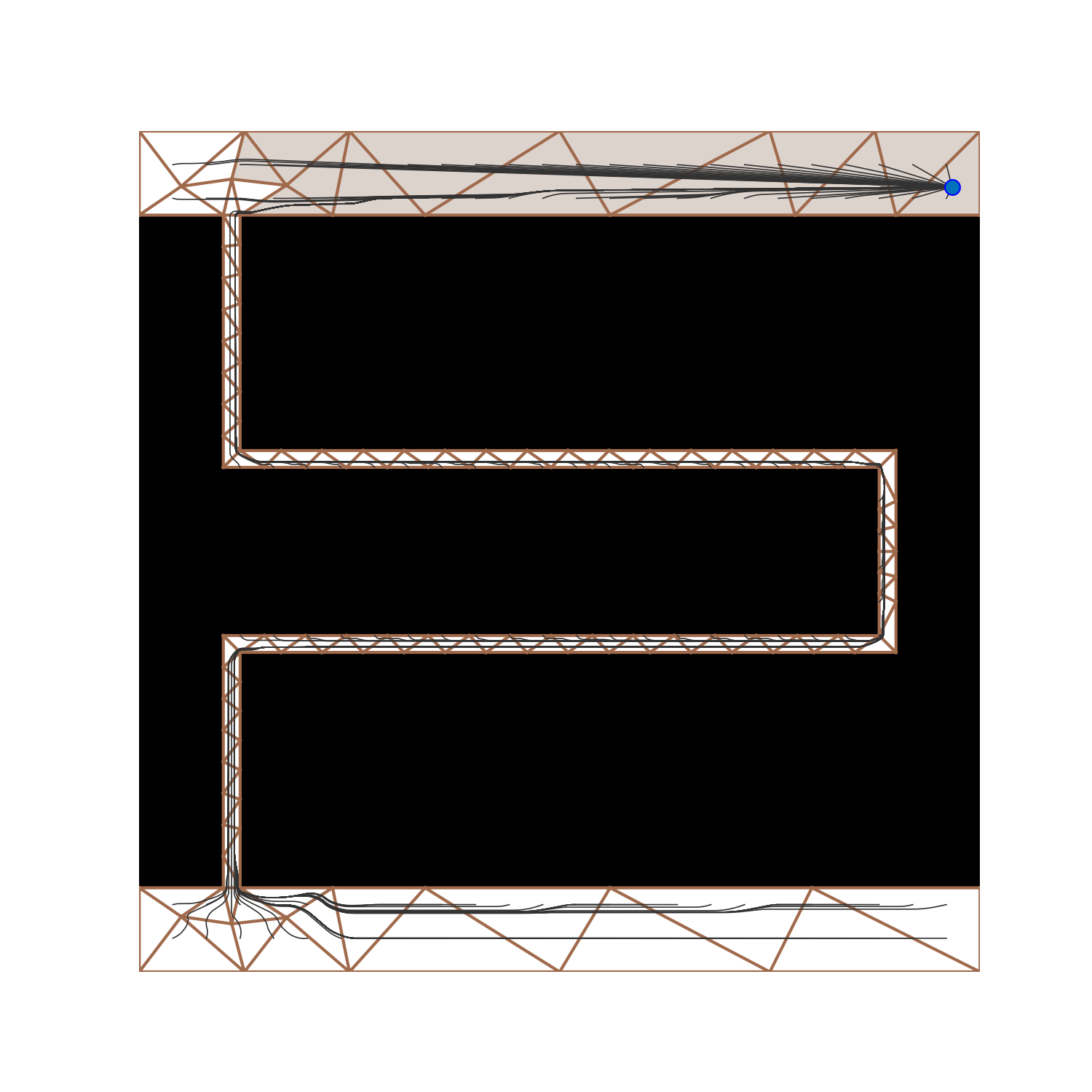}
        \caption{Proposed: Bug Trap}
        \label{fig:bug_proposed}
    \end{subfigure}

    \begin{subfigure}[b]{0.48\columnwidth}
        \centering
        \includegraphics[width=\textwidth,trim={2.5cm 2.5cm 2.5cm 2.5cm},clip]{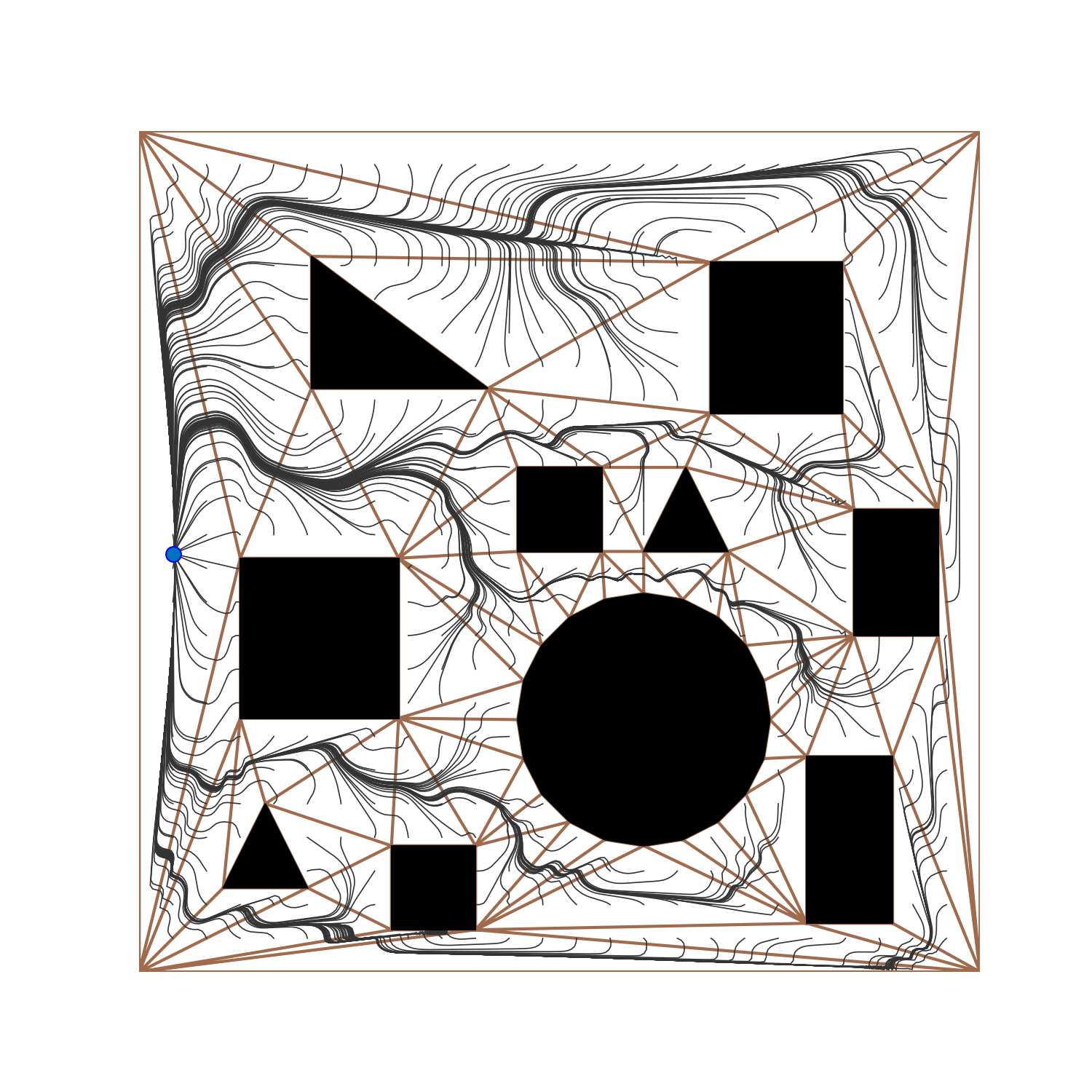}
        \caption{Baseline: Sparse}
        \label{fig:open_base}
    \end{subfigure}
    \hfill 
    \begin{subfigure}[b]{0.48\columnwidth}
        \centering
        \includegraphics[width=\textwidth,trim={2.5cm 2.5cm 2.5cm 2.5cm},clip]{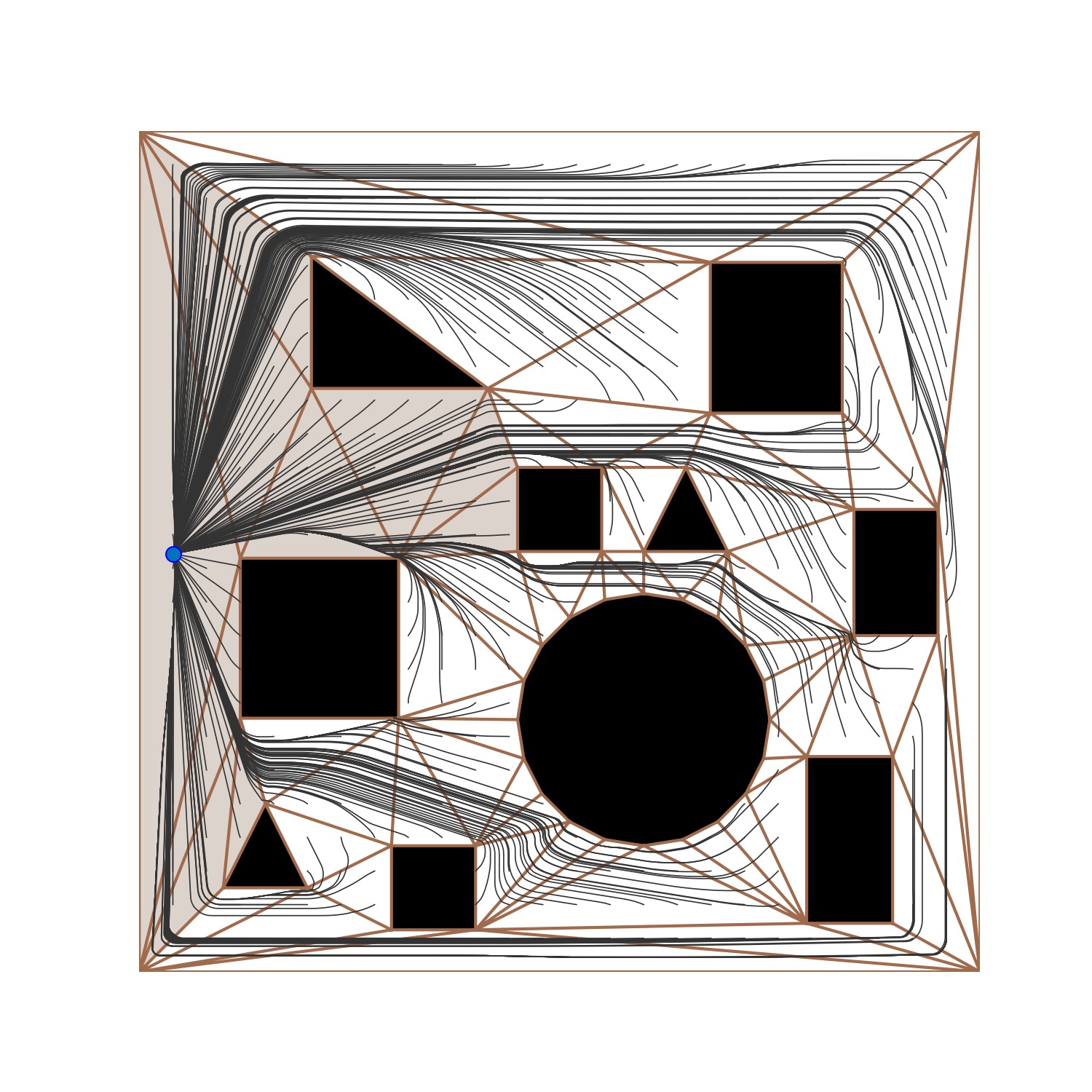}
        \caption{Proposed: Sparse}
        \label{fig:Sparse}
    \end{subfigure}

    \caption{Qualitative comparison of integral curves for the baseline and proposed methods across three environments. The maximal star-shaped region is highlighted in brown. Notably, in the Bug Trap environment, Steiner points were introduced to generate closer to equilateral triangles and avoid skewed triangles. }
    \label{fig:comparison_1x6}
\end{figure}

The results (Table~\ref{tab:results}, Figure~\ref{fig:comparison_1x6}) demonstrate a significant improvement over the baseline method. The primary goal of this work was to address the unnecessary trajectory bending inherent in arbitrary vector field assignments, a problem most directly captured by the total bending and total turning metrics. The proposed method achieved a remarkable 84.33-95.88\% reduction in total bending across all scenarios, with a win rate for total turning reaching 99.89\% in the Bug Trap environment. This geometric smoothness translated directly into improved dynamic performance, as evidenced by a 30.24-59.16\% reduction in LQR control effort, confirming that the generated paths are more energy-efficient and easier for a physical system to follow. These results strongly validate the effectiveness of the heuristic vector alignment and face-vector averaging techniques. 
While the proposed heuristic improves path quality, we cannot guarantee a lower maximum curvature in all cases, particularly when the underlying plan requires sharp changes in cell vector fields, as the heuristic's projection can itself induce a sharp turn.
\subsection{Ablation Study of Star-Shaped Region}
To quantify the specific contribution of the star-shaped region, we performed an ablation study evaluating a comprehensive set of goals (one per simplex) across all three environments, comparing our full method against using only heuristic alignment. The results indicate that while the star-shaped region depends on the goal position, environment, and simplicial complex, it plays an important role in generating smoother integral curves near the goal. Specifically, it reduced total bending by 74.35\% and total turning by 44.07\% in the Sparse environment, with a 34.77\% reduction of total bending in the Maze, and 20.01\% in the Bug Trap. This confirms that the star-shaped region effectively exploits local open space to generate more direct paths.

\subsection{Comparative Analysis}
We evaluated the proposed method against four baselines: RRT*~\cite{karaman2011sampling} (sampling-based), A*~\cite{hart1968formal} (grid-based), FPP~\cite{marcucci2024fast} (optimization-based), and APFs~\cite{khatib1986real} (feedback-based). To demonstrate practical utility in realistic scenarios, we used the Moving AI Lab benchmarks~\cite{sturtevant2012benchmarks}, specifically the ``Boston'' city map, alongside our Maze environment.   Baselines were implemented using PythonRobotics~\cite{sakai2018pythonrobotics} and the \href{https://github.com/cvxgrp/fastpathplanning}{\texttt{FPP}} library with quadtree decomposition. We evaluated more than 10 distinct start-goal configuration pairs for each environment. Due to its stochastic nature, RRT* was executed 10 times per pair (with a maximum of 5200 iterations), and failures were reported if no path was found. As detailed in Table~\ref{tab:comparison_results} and Figure~\ref{fig:comparison_methods}, the proposed method, A*, and FPP achieved 100\% success across both environments, whereas the APFs method struggled to find feasible paths. Although the proposed method can yield slightly longer paths, it operates in the sub-90 ms range, making it an efficient and robust feedback planner. We acknowledge that the performance of baseline methods could potentially be improved through further code optimization or extensive parameter tuning; however, these results reflect their performance using standard, widely accessible implementations.

\begin{table}[h!]
\centering
\caption{Quantitative comparison of different algorithms in two different environments. Results report the mean $\pm$ standard deviation.}
\label{tab:comparison_results}
\resizebox{\columnwidth}{!}{%
\begin{tabular}{lccccc}
\toprule
\textbf{Environment} & \textbf{Method}  & \textbf{Continuity} & \textbf{Success Rate (\%)} & \textbf{Total Time (s)} & \textbf{Path Length (m)} \\
\midrule
\multirow{5}{*}{\textbf{Sparse (Boston)}} 
 & APF  & $C^0$ & 0.0 & N/A & N/A \\
 & A*  & $C^0$ & 100.0 & $0.042 \pm 0.018$ & $9.823 \pm 1.797$ \\
 & RRT*  & $C^0$ & 100.0 & $0.826 \pm 0.696$ & $12.491 \pm 2.314$ \\
 & FPP   & $C^k$ & 100.0 & $3.718 \pm 0.939$ & 9.547 $\pm$ 1.684 \\
 & Proposed  & $C^\infty$ & 100.0 & 0.087 $\pm$ 0.003 & $10.280 \pm 2.219$ \\
\midrule
\multirow{5}{*}{\textbf{Maze}} 
 & APF &  $C^0$ & 0.0 & N/A & N/A \\
 & A* &  $C^0$ & 100.0 & $0.085 \pm 0.058$ & 14.74 $\pm$ 7.67 \\
 & RRT* &  $C^0$ & 80.9 & $0.788 \pm 0.874$ & $14.81 \pm 5.83$ \\
 & FPP  &  $C^k$ & 100.0 & $2.739 \pm 1.618$ & $15.23 \pm 8.06$ \\
 & Proposed &  $C^\infty$ & 100.0 & 0.033 $\pm$ 0.010 & $15.56 \pm 8.19$ \\
\bottomrule
\end{tabular}%
}
\end{table}

\begin{figure}[t!]
    \centering

    \begin{subfigure}[b]{0.49\columnwidth}
        \centering
        \includegraphics[width=0.85\textwidth,trim={0.2cm 0.0cm 0.2cm 0.0cm},clip]{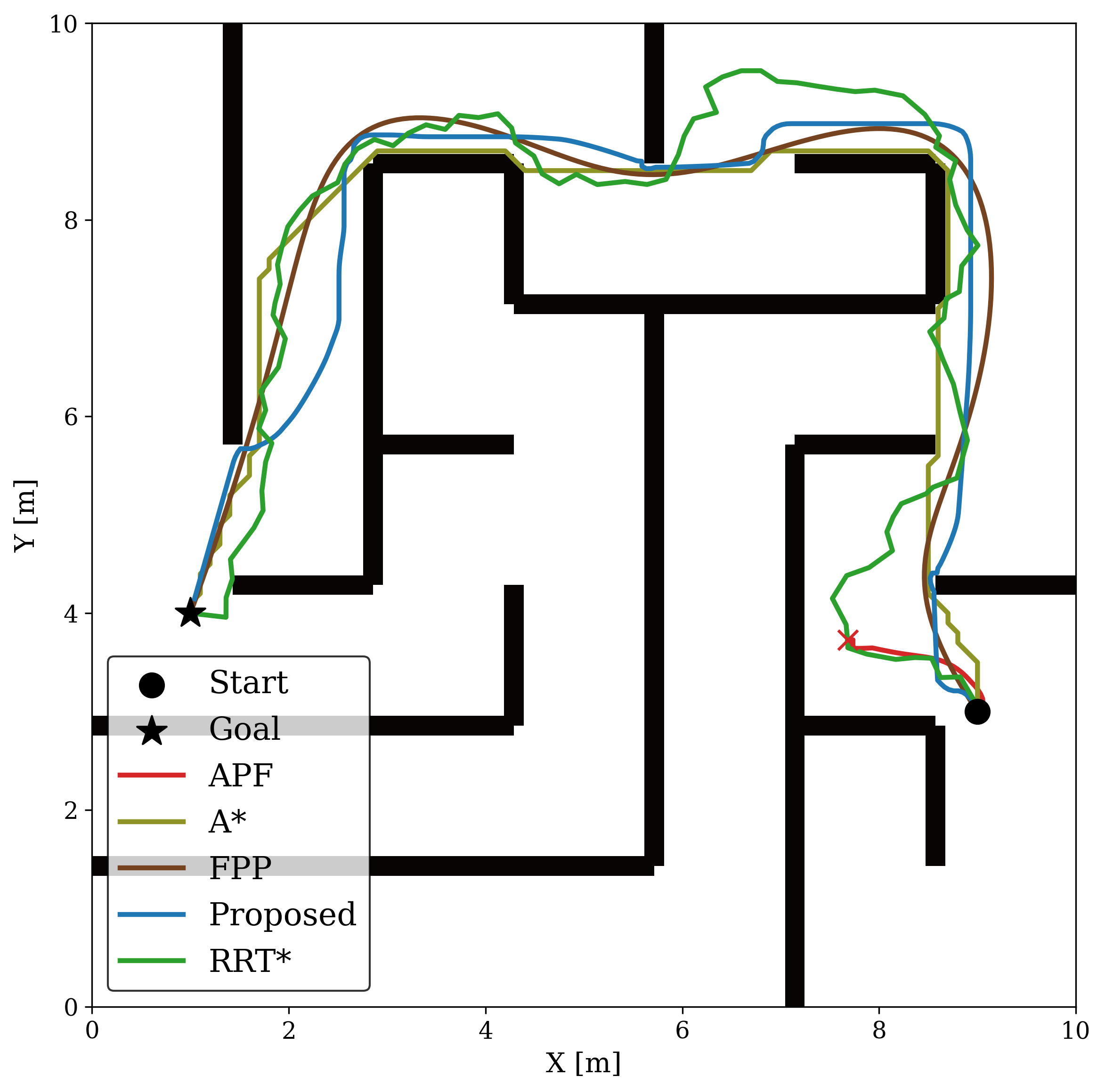}
        \caption{Maze}
    \end{subfigure}
    \hfill 
    \begin{subfigure}[b]{0.49\columnwidth}
        \centering
        \includegraphics[width=0.85\textwidth,trim={0.2cm 0.0cm 0.2cm 0.0cm},clip]{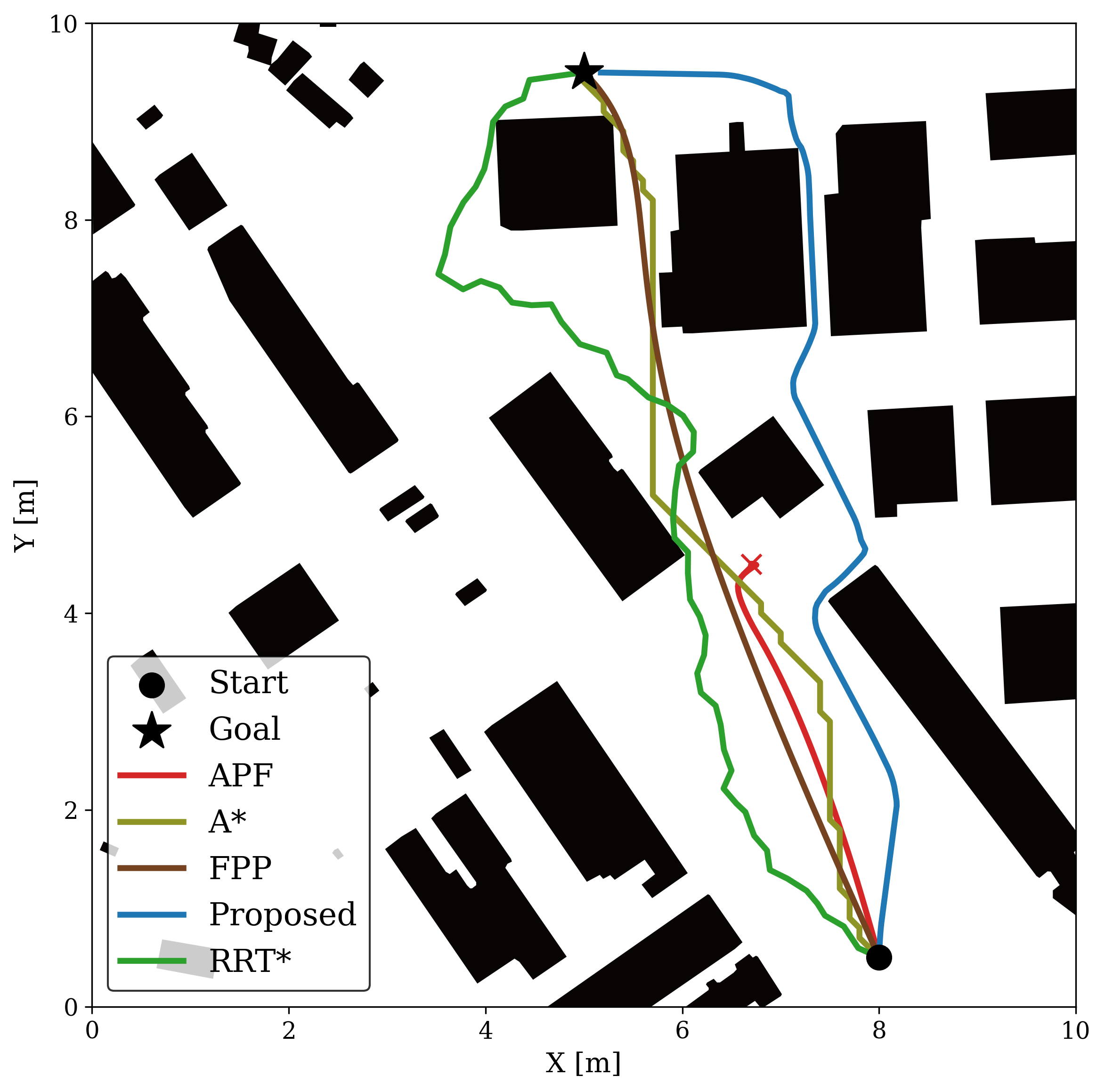}
        \caption{Sparse (Boston)}
    \end{subfigure}

    \caption{Qualitative comparison of paths for the baselines across the maze and sparse environments.}
    \label{fig:comparison_methods}
\end{figure}

\subsection{Computational Efficiency} 
Our method splits the computation into offline and online phases for efficiency. The offline phase includes simplicial decomposition, which in 2D (for polygonal obstacles and by CDT), can be performed in $O(n \log n)$ time, where $n$ is the number of obstacle vertices~\cite{shewchuk1996triangle}. 

A discrete plan can be computed via Dijkstra's algorithm on the connectivity graph. The successor of each simplex can be found in $O(m\log m)$ time, in which $m$ is the number of $d$-simplexes. However, recent theoretical advances have produced even faster algorithms for this problem~\cite{duan2025breaking}. Heuristic alignment and star-shaped funnel construction add minimal linear overhead via local geometric checks per simplex. Online execution is lightweight. If the robot's current cell is unknown, a point-location query can be performed in logarithmic time (in 2D) with preprocessing, but can be more expensive in higher dimensions~\cite{toth2017handbook, lindemann2009simple}. In a given $d$-simplex, the computation is minimal, requiring only distance checks to its $d+1$ faces followed by the vector field blending.

On a system with an AMD Ryzen 9 8945HS CPU and 32\,GB of RAM, using the \href{https://www.cs.cmu.edu/~quake/triangle.html}{\texttt{triangle}} library for CDT~\cite{shewchuk1996triangle}, the offline precomputation (including triangulation, discrete plan, heuristic alignment, and funnel construction) required about 6-9\,ms for the standard environments (Maze, Bug Trap, Sparse). For larger and more complex environments, such as the ``Boston'' benchmark, the precomputation time naturally scales with the number of obstacle vertices and the resulting simplexes. Although the implementation could be further optimized, for planar navigation, the system can handle dynamic environments via global recomputation at about 100 Hz. The online execution is lightweight; evaluating the vector field $V(x)$ at any point $x$ took on average about 0.07\,ms, which allows the robot to find a feedback policy in real time (about $15$ kHz update rate). We acknowledge that for higher-dimensional $\C $, precomputation time would increase, and local mesh repair (if the whole environment does not change) can be a valuable approach.

\section{Conclusion and Future Work}
We introduced a computationally efficient method to reduce unnecessary bending in feedback motion plans on a simplicial complex embedded in $\Cfree$. We acknowledge that computing a simplicial complex in higher dimensions can be intractable. However, for $d \le3$, algorithms like CDT (e.g. triangle~\cite{shewchuk1996triangle} for 2D, Tetgen~\cite{10.1145/2629697} for 3D) exist, making this method effective and tractable for mobile robots and UAVs. Our approach combines a novel heuristic for aligning local vector fields with a geometric algorithm that constructs a maximal star-shaped funnel around the goal, resulting in more efficient trajectories with significantly lower bending. A natural extension, forming a promising direction for future research, is to cover the entire free space with a set of such star-shaped regions, each with its own local waypoint. This could lead to a highly efficient, global feedback planning framework. Further work will also focus on adapting the alignment heuristics for systems with nonholonomic constraints and validating the framework on physical robot platforms. Furthermore, a promising direction involves co-optimizing the discrete cell-path with our vector field alignment to achieve more globally optimal trajectories.

\section{Acknowledgment}
We thank Hannah Erickson and Ba\c{s}ak Sak\c{c}ak for helpful discussions and feedback.

\bibliographystyle{IEEEtran}
\bibliography{sample-base}
%\addtolength{\textheight}{-12cm} 

\end{document}